\documentclass[twoside]{article}

\usepackage{microtype}
\usepackage{graphicx}
\usepackage{subcaption}
\usepackage{booktabs} 
\usepackage{amsfonts}
\usepackage{algorithm}
\usepackage{url}
\usepackage{algpseudocode}
\usepackage[table]{xcolor}
\usepackage{multirow}
\usepackage{csquotes}
\usepackage{bm}
\usepackage{enumitem}
\usepackage{tikz}
\usepackage{dblfloatfix}

\usepackage[round]{natbib}

\usepackage{hyperref}

\usepackage[accepted]{aistats2026}

\usepackage{amsmath}
\usepackage{amssymb}
\usepackage{mathtools}
\usepackage{amsthm}

\usepackage[capitalize,noabbrev]{cleveref}
\theoremstyle{plain}

\newtheorem{proposition}{Proposition}[section]

\theoremstyle{definition}

\theoremstyle{remark}

\newcommand{\params}{\boldsymbol{\theta}}
\newcommand{\data}{\mathcal{D}}

\usepackage[textsize=tiny]{todonotes}

\begin{document}

\twocolumn[

\aistatstitle{Towards E-Value Based Stopping Rules\\ for Bayesian Deep Ensembles}

\aistatsauthor{Emanuel Sommer\\LMU Munich, MCML \And Rickmer Schulte\\LMU Munich, MCML \And Sarah Deubner\\LMU Munich, MCML \AND Julius Kobialka\\LMU Munich, MCML  \And David R\"ugamer\\LMU Munich, MCML}
\aistatsaddress{}]

\begin{abstract}
  Bayesian Deep Ensembles (BDEs) represent a powerful approach for uncertainty quantification in deep learning, combining the robustness of Deep Ensembles (DEs) with flexible multi-chain MCMC. While DEs are affordable in most deep learning settings, (long) sampling of Bayesian neural networks can be prohibitively costly. Yet, adding sampling after optimizing the DEs has been shown to yield significant improvements. This leaves a critical practical question: \textit{How long should the sequential sampling process continue to yield significant improvements over the initial optimized DE baseline?}
  To tackle this question, we propose a stopping rule based on E-values. We formulate the ensemble construction as a sequential anytime-valid hypothesis test, providing a principled way to decide whether or not to reject the null hypothesis that MCMC offers no improvement over a strong baseline, to early stop the sampling. Empirically, we study this approach for diverse settings.
  Our results demonstrate the efficacy of our approach and reveal that only a fraction of the full-chain budget is often required.
\end{abstract}

\section{INTRODUCTION} 

Bayesian deep learning constitutes a principled framework for quantifying uncertainty in modern neural networks \citep{pmlr-v235-papamarkou24b}. Among various approximations, Bayesian Deep Ensembles (BDEs) have emerged as a particularly promising method, effectively bridging the gap between the practical performance of Deep Ensembles \citep[DE;][]{lakshminarayanan_2017_SimpleScalablea} and the rigor and flexibility of MCMC \citep{sommer2024connecting, sommer2025mile, sommer2025can, duffield2025scalable, paulin_sampling_2025}. The general recipe for BDE ensembles involves multiple short MCMC chains initialized in high-likelihood regions identified by individual DE members. Since the optimized solution already incurs non-negligible computational cost, this raises a natural question: how much additional sampling is truly necessary to improve ensemble performance? Put differently: \textit{what is the minimal BDE required to deliver significant gains over a strong baseline such as a Deep Ensemble?} 

Existing practice typically relies on running chains for a predetermined number of iterations or applying ad hoc early-stopping heuristics, for example, by monitoring validation metric plateaus \citep{sommer2024connecting}. Although often effective in practice, such strategies lack a principled statistical foundation, provide no formal guarantees under adaptive stopping, and may lead to unnecessarily long sampling runs.
We address this gap by introducing an E-value-based stopping rule. E-values quantify evidence against a null hypothesis in a way that remains valid under optional stopping, making them well suited for sequential and anytime decision procedures \citep{Ramdas2024HypothesisTW}.

Our contributions are as follows:
\begin{itemize}
    \item We formalize the BDE sampling process as a sequential hypothesis test using E-values, providing theoretical conditions under which our stopping rule constitutes a valid sequential test based on a test supermartingale.
    \item We demonstrate on diverse tasks that our stopping criterion identifies minimal BDEs that achieve competitive performance while reducing a naive sample budget by a factor of $10-60$. 
    \item The resulting approach 1) provides practical guidance for distillation and compression approaches and 2) autonomously abstains from collecting samples and prevents wasteful computation when the MCMC chain fails to provide statistical evidence of improvement.
\end{itemize}

It is important to note that we do not claim to find the \textit{optimal} ensemble size (which may be infinitely large); rather, we identify the point at which we can reject the hypothesis that the extra computation invested in MCMC sampling did not yield significant improvements.

\section{BACKGROUND} 

We consider a supervised learning setting with training data $\data = \{(\bm{x}_i,y_i)\}_{i=1}^{n} \in (\mathcal{X} \times \mathcal{Y})^n$ and a holdout validation dataset $\data_{\text val} = \{(\bm{x}_i, y_i)\}_{i=1}^m$. Using a Bayesian neural network, we aim to approximate the posterior predictive density (PPD)
\begin{equation*}\label{eq:PPD}
    p({y} | \bm{x}, \data) = \int p({y}|\bm{x}, \params) p(\params|\data) d\params \approx \frac{1}{K} \sum_{j=1}^{K} p({y} | \bm{x}, \params_j)
\end{equation*}
using $K$ posterior samples $\{\params_j\}_{j=1}^K$, where $p(y|\bm{x}, \params_j)$ is the likelihood for data point $(y,\bm{x})\in \mathcal{X}\times\mathcal{Y}$ 
, parameterized by the $j$th posterior sample $\params_j$.
In the BDE context, $\params_1$ represents a strong reference model (e.g., an optimized solution via Adam optimization or the first sample after the samplers burn-in). Subsequent samples $\{\params_2, \dots, \params_K\}$ are generated sequentially using a sampling algorithm specifically designed for Bayesian neural networks (e.g., pSMILE, \citealp{sommer2025can}). While previous research addresses optimal BDE chain initialization under a limited computational budget \citep{rundel2025efficiently}, we focus specifically on the sampling phase. Throughout this work, we employ isotropic Gaussian priors on $\params$ \citep{kobialka2026interplay}.

\paragraph{E-values}
An E-value is a non-negative random variable $E$ such that $\mathbb{E}_{P_0}[E] \leq 1$ under a null hypothesis $P_0$. Unlike p-values, E-values allow for the accumulation of evidence over time without multiple testing corrections. A sequence of E-values $(E_t)_{t \in \mathbb{N}}$ forms an E-process if, for any stopping time $\tau \in \mathbb{N}$, $\mathbb{E}_{P_0}[E_\tau] \leq 1$ \citep{Ramdas2024HypothesisTW}. This property is crucial for our setting, as it permits the practitioner to stop collecting samples as soon as sufficient evidence is obtained.

\section{E-VALUE AS STOPPING CRITERION FOR BDE}

To obtain a principled stopping rule, we cast the problem as testing whether the additional sampling in hybrid approaches such as BDEs yields a statistically significant improvement in predictive performance over the optimized baseline. We consider the case where the validation set $\data_{\text{val}}$ is fixed, and the source of randomness is the sequence of parameters generated by the sampling algorithm.

\subsection{Hypothesis Definition}
Let $L(\params) = \prod_{(\bm{x},y) \in \data_{\text val}} p(y|\bm{x}, \params)$ denote the likelihood of the validation set for a single parameter vector $\params$. Let $\params_1$ be the optimized warm-start baseline. 

Since $\params_1$ usually constitutes a well-performing single-mode approximation, we can cast the question of whether sampling improves upon $\params_1$ as the following null hypothesis $H_0$ with underlying distribution $Q$:
\begin{equation}
    H_0: \mathbb{E}_{\params \, \overset{\text{i.i.d.}}{\sim} Q}[L(\params)] \le L(\params_1).
\end{equation}
Under $H_0$, the sampling algorithm generates models that, on average, have a likelihood on the validation data no better than the baseline model. 

From a purely Bayesian perspective, one might object to this formulation, noting that in high-dimensional parameter spaces the typical set, where most posterior mass concentrates, can exhibit lower pointwise likelihood than the mode \citep{betancourt_2018_ConceptualIntroductiona}. Consequently, a sampler might naturally drift into lower likelihood regions. 
However, in the context of BDEs, our objective is to efficiently collect high-utility, diverse models. We posit that, under a finite ensemble budget, transitioning to regions with substantially lower validation likelihood---even if theoretically justified by posterior volume---is computationally inefficient from the perspective of predictive performance. By aiming for samples to maintain parity with or improve upon the baseline mode $\params_1$, we bias the selection towards high-density regions and high-performing modes. This view is further supported by the discussion provided in App.~\ref{app:connectionjensen}.

\subsection{E-value Validity and Construction} \label{sec:methods}

For each sample $\params_k$ ($k > 1$), we define the individual E-value $S_k$ as the likelihood ratio on the validation set:
\begin{equation}
    S_k = \frac{L(\params_k)}{L(\params_1)} = \prod_{i=1}^m \frac{p(y_i|\bm{x}_i, \params_k)}{p(y_i|\bm{x}_i, \params_1)}.
\end{equation}

\begin{proposition}
    Under $H_0$ and canonical filtration $\mathcal{F}_k = \sigma(\theta_1, \dots, \theta_k)$, the statistic $S_k$ is a valid e-variable, satisfying $\mathbb{E}_{H_0}[S_k] \le 1$.
\end{proposition}
\begin{proof}
    Since $\params_1$ is fixed, $L(\params_1)$ is a positive constant. For any distribution $Q$ satisfying $H_0$:
    \begin{equation*}
        \mathbb{E}_{H_0}[S_k] = \mathbb{E}_{\params_k \sim Q}\left[\frac{L(\params_k)}{L(\params_1)}\right] = \frac{\mathbb{E}_Q[L(\params_k)]}{L(\params_1)} \le 1,
    \end{equation*}
    where the inequality follows directly from the definition of~$H_0$ in~(1). Since $S_k \ge 0$ by construction, it is a valid e-variable. Specifically, we are testing the composite null that the sampling distribution $Q$ does not improve expected likelihoods relative to $\params_1$.
\end{proof}

In practice, we accumulate evidence over the sequence of samples and define the accumulated evidence process $E_k$ as

$$
    E_k = E_{k-1} \times S_k = \textstyle \prod_{j=2}^k \frac{L(\params_j)}{L(\params_1)}
$$
with $E_1 = 1$.
\begin{figure*}[t]
    \centering
    \begin{subfigure}[t]{0.48\textwidth}
        \centering
        \includegraphics[width=\linewidth]{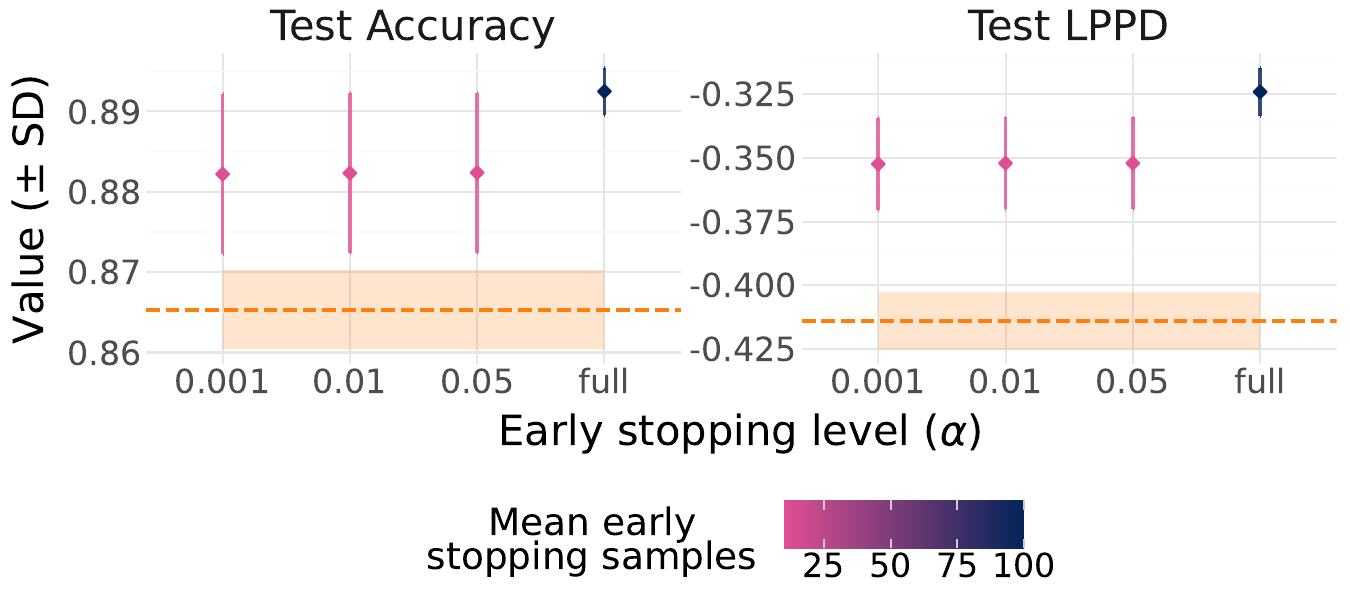}
        \caption{Chainwise performance}
    \end{subfigure}%
    ~ 
    \begin{subfigure}[t]{0.48\textwidth}
        \centering
        \includegraphics[width=\linewidth]{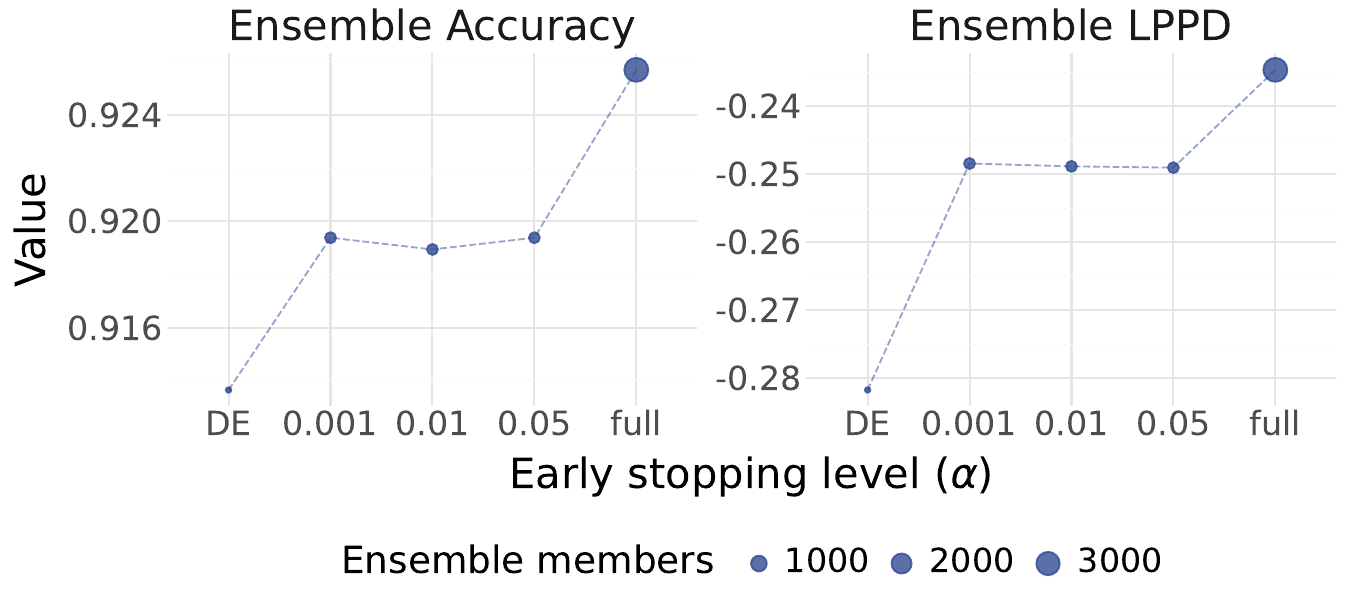}
        \caption{Ensemble performance}
    \end{subfigure}
    \caption{Chainwise and ensemble hold-out test performance of the E-value induced minimal BDEs in comparison to the full ensemble and the optimized warmstarts, i.e., the DE (members). The reference baseline is the first posterior sample \textbf{after} the sampler's warmup. The architecture is a ResNet-7 on the \texttt{CIFAR-10} dataset. Error bars represent standard deviation over the 32 chains/members and the optimized warmstart baseline performance is provided as dashed orange line (standard deviation as shaded area).}
    \label{fig:performance_resnet7_firstsample}
\end{figure*}

\begin{table}[b!]
\centering
\caption{Chainwise and ensemble hold-out test performance along with the compression and sizes of the E-value induced minimal BDEs in comparison to the full ensemble and the optimized warmstarts, i.e., the DE (members). The task is a 22M parameter ViT on the \texttt{Imagenette} dataset with 16 chains/members. Reference used for the E-value is given in brackets.}
\label{tab:compression_vit}
\vspace{2mm}
\addtolength{\tabcolsep}{-3pt} 
\small 
\begin{tabular}{l ccc}
\toprule
\textbf{Method} & \textbf{LPPD} ($\uparrow$) & \textbf{Samples} & \textbf{Compr.} \\
\midrule
\textit{Ensemble Performance} & & & \\
DE & -0.9053 & 16 & $\times100$ \\
$\alpha=0.01$ (DE ref.) & -0.7553 & 16 & $\times100$\\
$\alpha=0.01$ (Sample ref.) & -0.7950 & 158 & $\times10.1$\\
Full BDE & -0.7594 & 1600 & $\times1$ \\
\midrule
\textit{Single Chain (Mean)} & & & \\
DNN  & -1.6107 & 1 & $\times100$ \\
$\alpha=0.01$ (DE ref.) & -1.3912 & 1 & $\times100$\\
$\alpha=0.01$ (Sample ref.) & -1.3110 & 9.9 & $\times10.1$\\
Full Chain & -1.2964 & 100 & $\times1$ \\
\bottomrule
\end{tabular}
\addtolength{\tabcolsep}{3pt} 
\end{table}

\paragraph{On the Assumption of Effective Independence} The construction of $E_k$ as a test supermartingale formally relies on the independence of the likelihood ratios $S_k$. In practice, this can be ensured via sufficient thinning and MCMC diagnostics. A detailed discussion can be found in App.~\ref{app:independence}.

By Proposition 7.20 in \cite{Ramdas2024HypothesisTW}, under this independence assumption, the sequence $(E_k)_{k \ge 1}$ constitutes a test supermartingale. This enables anytime-valid stopping via Ville’s Inequality \citep{ville1939collectif}:
\begin{equation}
P_{H_0}(\exists k : E_k \ge 1/\alpha) \le \alpha
\end{equation}
for any significance level $\alpha\in(0,1).$ In other words, we stop sampling as soon as $E_k \ge 1/\alpha$. If the threshold is not met within the maximum budget, the chain is discarded for lack of statistical evidence.

\section{EMPIRICAL RESULTS}
\label{sec:results}

We investigate the practical utility of E-value based stopping for constructing ``minimal BDEs." We define a minimal BDE as the ensemble constructed by stopping the MCMC chains the moment the accumulated evidence $E_k$ rejects the null hypothesis $H_0$ at a significance level $\alpha$. If $H_0$ is rejected, we have statistical confidence that the sampling process is generating models that outperform the baseline; if not, the method signals that additional computational investment might not yield substantial returns.

We evaluate this across architectures and modalities: a Vision Transformer (ViT) on \texttt{Imagenette}, a ResNet-7 on \texttt{CIFAR-10}, and Multi-Layer Perceptrons (MLP) on the \texttt{bikesharing} (regression) and \texttt{income} (binary classification) datasets. We compare two reference baselines ($\params_1$) for the test: (1) the optimized warmstart (DE), testing if sampling improves upon the AdamW solution, and (2) the first posterior sample, testing if the chain improves upon its post-warmup state.

\subsection{Minimal BDEs in Vision Transformers}

We first analyze the ViT. \cref{fig:performance_vit_firstsample,fig:performance_vit_de} (Appendix) illustrate the predictive performance of the minimal BDE on \texttt{Imagenette} for various $\alpha$ levels and for the two considered baselines respectively. We observe that the minimal BDE guided by E-values in all cases achieves a LPPD competitive with the full ensemble, despite utilizing only a fraction of the computational budget.

As detailed in \cref{tab:compression_vit}, utilizing the \textit{first sample} reference allows for a rejection of $H_0$ with roughly 10 samples per chain, resulting in a $10$-fold compression compared to the fixed budget of 100 samples per chain. While the individual chain accuracy is slightly lower than that of the full chain, the ensemble accuracy of the minimal BDE is boosted significantly, effectively matching or slightly surpassing the full BDE. This indicates that the full chains are likely sampling excessively; the E-value criterion successfully identifies a point where the ensemble has captured sufficient diversity to meaningfully improve upon the baseline. In particular, this ensemble size can serve as a reasonable target size for post-hoc distillation and compression algorithms.

When using the stricter \textit{DE reference} (comparing against the optimized mode), the stopping criterion triggers even earlier (\cref{tab:compression_vit}), yielding ensembles that slightly outperform the full baseline in LPPD with minimal members. Considering the non-negligible warmup of the sampler, this earlier stopping is intuitively clear and found also in later experiments.

\subsection{Robustness and Generalization}

\paragraph{ResNet-7 (CIFAR-10)} Consistent with the ViT results, we observe high compression rates (with factors of $10-60$, see Table \ref{tab:compression_resnet} in Appendix). However, unlike the ViT, the minimal BDE does not fully match the predictive performance of the full chain (see \cref{fig:performance_resnet7_de,fig:performance_resnet7_firstsample}). This highlights the nuanced trade-off: our test detects significant improvement over the baseline, but does not guarantee optimality. The result suggests that while the first few high-quality posterior samples provide the majority of improvement over the mode, the ``long tail" of the chain still contributes to the posterior approximation.

\paragraph{Distributional Regression} On the tabular regression task, our method offers the largest efficiency gains compared to the default setting proposed, e.g., in \citet{sommer2025mile}. For the \textit{first sample reference}, the minimal BDE reduces the number of samples by a factor of $\approx 250$, stopping after $\approx4$ samples per chain, while still outperforming the DE baseline in both RMSE and LPPD (see \cref{tab:compression_mlpbike,fig:performance_mlpbike_de,fig:performance_mlpbike_firstsample}). This highlights the method's generalizability beyond classification.

\paragraph{Sanity Check} On the \texttt{income} classification task (Table \ref{tab:compression_mlpincome}, Appendix), we observe that with the \textit{DE reference}, the E-values never cross the rejection threshold $1/\alpha$ for all considered $\alpha$-levels. This result is noteworthy as it indicates that, within the considerable yet constrained computational budget, the sampler could not statistically improve upon the DE. Rather than forcing the collection of a minimal number of marginally useful samples, the method defaults to the DE (0 extra samples), preventing wasteful computation at inference.

\subsection{Dynamics of Evidence Accumulation}

A critical characteristic and a potential limitation of this formulation is the high volatility of the test statistic. Because the E-value $E_k$ is a product of likelihood ratios, small but consistent differences in model quality translate into exponential growth or decay. As illustrated in \cref{fig:evalue_dynamics}, this leads to a fanning out effect. The chains that explore higher-density regions reject the null hypothesis quite rapidly, while those that degrade performance vanish towards zero just as quickly.

\begin{figure}[h]
    \centering
    \includegraphics[width=0.8\linewidth]{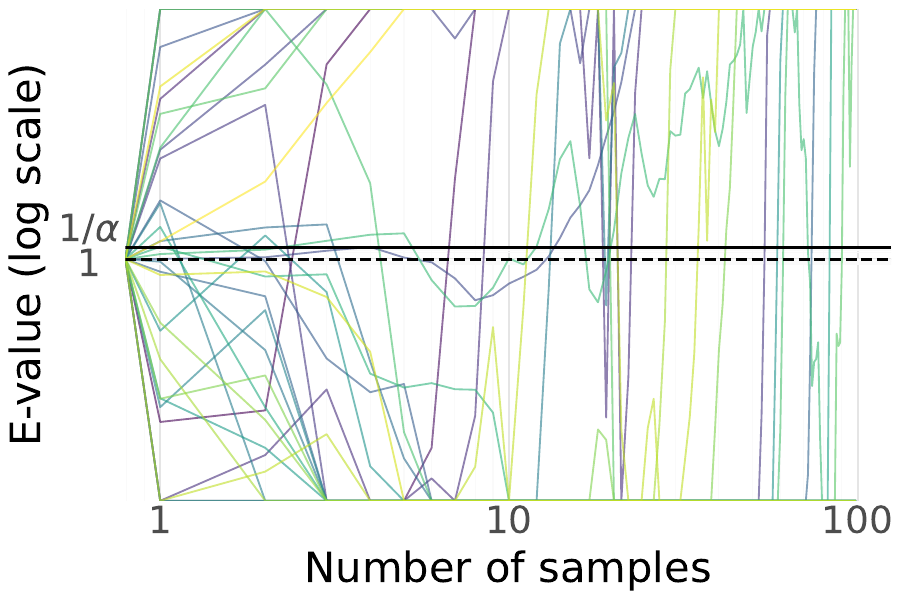}
    \caption{Evolution of E-values for ResNet-7 chains on \texttt{CIFAR-10} (Log-scale). Each colored line represents a single MCMC chain (32 in total). The rejection threshold for $\alpha=0.01$ is given as a black solid line.}
    \label{fig:evalue_dynamics}
\end{figure}

\section{CONCLUSION}

By formalizing MCMC sampling as a sequential hypothesis test using E-values, we show that it is possible to identify minimal ensembles that capture the majority of predictive performance with $10-60$ times less compute. This framework introduces a principled mechanism for autonomous abstention, ensuring resources at inference are only spent when the chain statistically improves upon the chosen baseline. 

As we are taking the first steps toward a decision-theoretic framework for Bayesian Deep Ensembles, our approach remains sensitive to the volatility of likelihood ratio-based E-values and validation set noise. Furthermore, while thinning approximates the independence required for our martingale guarantees, persistent autocorrelation remains a theoretical challenge. Moreover, the current formulation focuses on baseline improvement, incorporating a fixed target outperformance ratio, for example, calibrated to the DE variability, is a natural extension. The stopping rule is presently defined for a single chain, with multi-chain extensions via compositional E-value properties left for future work. We view this as a starting point for developing more sophisticated, dependence-aware E-processes for BDEs; more broadly, related paradigms such as subspace inference \citep{izmailov_2020_SubspaceInferencea, dold2024, dold25paths} may also benefit from these ideas.

\section*{Acknowledgments}

This research was supported by grants from NVIDIA and utilized NVIDIA products, mainly NVIDIA A100 GPUs for the posterior sampling.

\bibliography{refs}
\bibliographystyle{apalike}

\newpage
\appendix
\onecolumn

\section{CONNECTION OF AVERAGE SAMPLE PERFORMANCE TO POSTERIOR PREDICTIVE}\label{app:connectionjensen}

While our statistical test operates on the sequence of individual samples $\params_k$, it serves as a rigorous proxy for the quality of the Ensemble PPD. 

Empirically, the PPD is approximated as the average predictive distribution. By Jensen's Inequality, the log pointwise predictive density (LPPD), which is advocated by \citet{gelman2014a} as the metric of choice for evaluating the quality of the full predictive posterior, is lower-bounded by the average log-likelihood of its members:
\begin{equation}
    \text{LPPD} = \sum_{(\bm{x},y)} \log \left( \frac{1}{K} \sum_k p(y|\bm{x}, \params_k) \right) \ge \frac{1}{K} \sum_k \sum_{(\bm{x},y)} \log p(y|\bm{x}, \params_k)
\end{equation}
Our proposed stopping criterion tests the right-hand side of this inequality (the average member performance). Therefore, if our E-value rejects $H_0$, it guarantees that the member models have improved over the baseline significantly. Due to the inequality, this implies the PPD has improved by \textit{at least} that margin, plus any additional gains from ensemble diversity. This makes our criterion a valid, yet conservative test for PPD improvement.

\section{DETAILED DISCUSSION ON INDEPENDENCE AND MCMC AUTOCORRELATION} \label{app:independence}

The construction of $E_k$ as a test supermartingale formally relies on the independence of the likelihood ratios $S_k$. While raw MCMC chains exhibit autocorrelation, strict independence is not required for the utility of the test, provided the chain achieves effective independence via sufficient thinning. Recent analysis of the empirical posteriors of BDEs suggests that while global mixing is difficult, local mixing is achievable with state-of-the-art samplers. For example, \citet{sommer2024connecting,sommer2025mile} demonstrate that BDEs generated via samplers like pSMILE or MILE achieve rapid local stationarity, evidenced by very low $c\hat{R}$ values. Furthermore, the momentum decoherence rate in these MCLMC-based samplers serves as a proxy for the necessary thinning interval; by setting the thinning rate larger than the autocorrelation time, subsequent samples $\params_k$ and $\params_{k+1}$ become effectively statistically independent for the purpose of the test. We explicitly note that if significant autocorrelation persists (e.g., due to insufficient thinning or locally trapped chains), the product of positively correlated E-values may grow faster than theoretically warranted, potentially inflating the Type-I error rate. The guarantees in Section \ref{sec:methods} strictly hold for independent streams; in the presence of dependence, the stopping rule should be interpreted as approximate.

\section{EXPERIMENTAL SETUP}

As benchmark datasets, we use \texttt{bikesharing} \cite{misc_bike_sharing_dataset_275}, \texttt{income} \citep{kohavi96incomedata}, CIFAR-10 \citep{krizhevsky2009learning} and \texttt{Imagenette} \citep{imagenette_dataset} and deploy the codebase of \citet{sommer2025can} to generate the full BDE ensemble/set of samples. We provide all detailed experimental configurations in our codebase available at \href{}{https://github.com/EmanuelSommer/eBNN}.

\clearpage
\section{ADDITIONAL RESULTS}\label{app:additional_results}

All additional results for \cref{sec:results} can be found below. In particular the analysis for the different reference choice.

\subsection{Additional Results for the ViTs}

\begin{figure*}[h]
    \centering
    \begin{subfigure}[t]{0.48\textwidth}
        \centering
        \includegraphics[width=\linewidth]{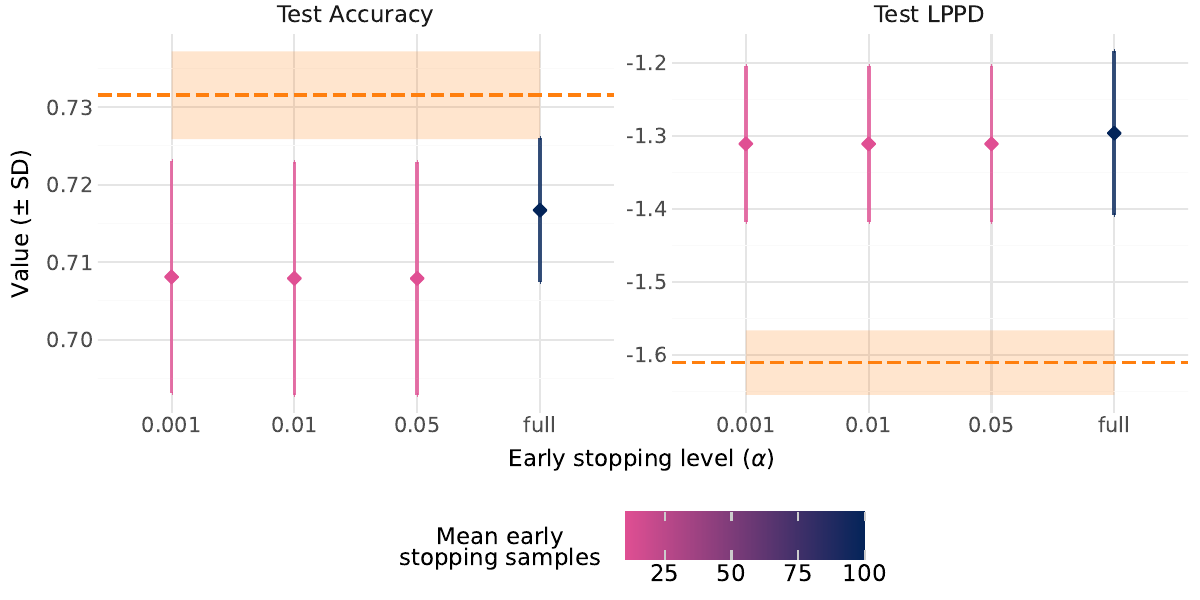}
        \caption{Chainwise performance (error bars represent standard\\deviation over 16 chains, optimized warmstart in orange)}
    \end{subfigure}%
    ~ 
    \begin{subfigure}[t]{0.48\textwidth}
        \centering
        \includegraphics[width=\linewidth]{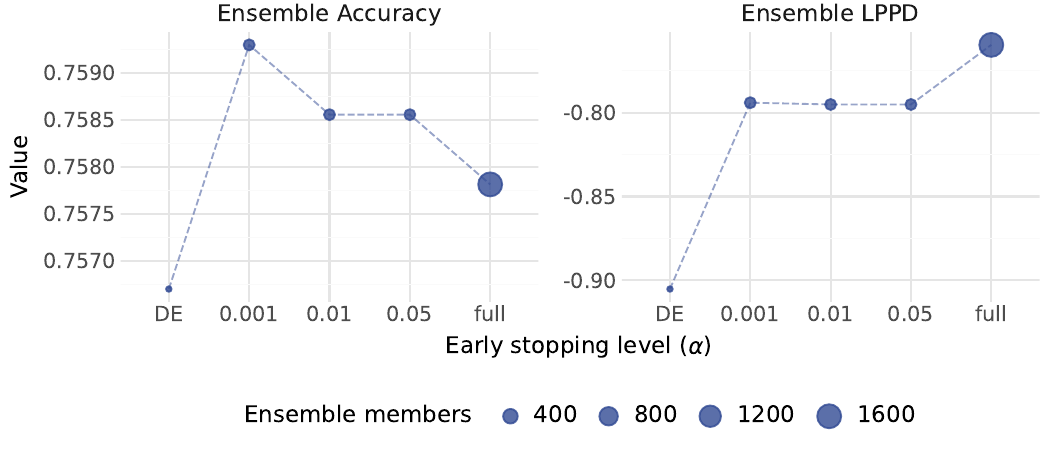}
        \caption{Ensemble performance (16 ensemble members)}
    \end{subfigure}
    \caption{Chainwise and ensemble hold-out test performance of the E-value induced minimal BDEs in comparison with the full ensemble and the optimized warmstarts i.e. the DE (members). The reference baseline is the first posterior sample \textbf{after} the sampler's warmup. The architecture is a 22M parameter ViT on the \texttt{Imagenette} dataset.}
    \label{fig:performance_vit_firstsample}
\end{figure*}

\begin{figure*}[h]
    \centering
    \begin{subfigure}[t]{0.48\textwidth}
        \centering
        \includegraphics[width=\linewidth]{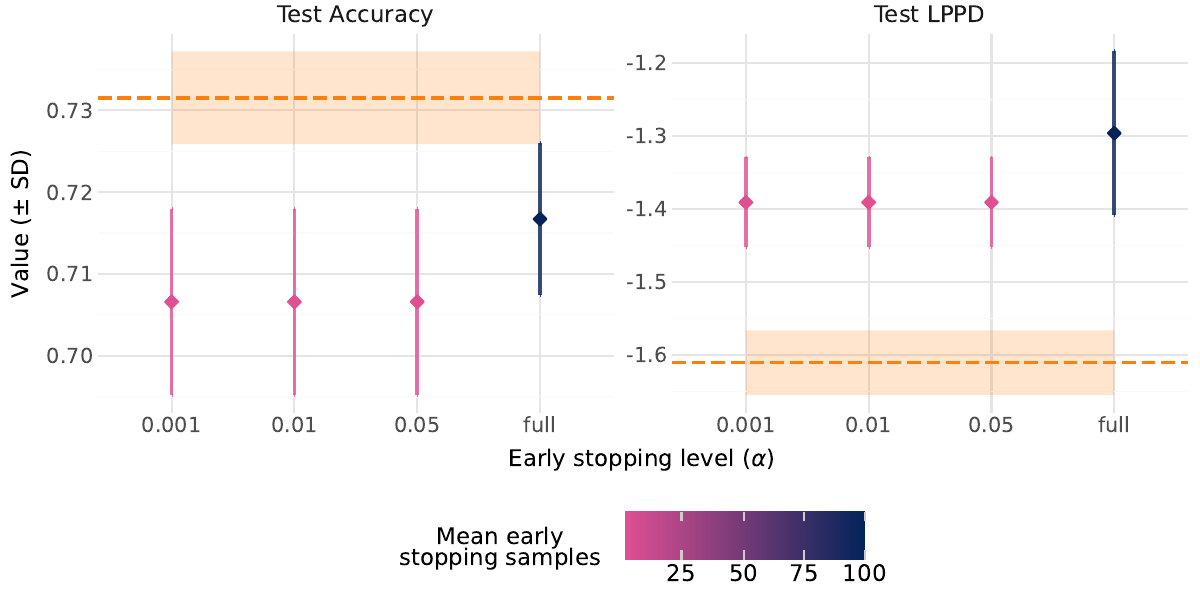}
        \caption{Chainwise performance (error bars represent standard\\deviation over 16 chains, optimized warmstart in orange)}
    \end{subfigure}%
    ~ 
    \begin{subfigure}[t]{0.48\textwidth}
        \centering
        \includegraphics[width=\linewidth]{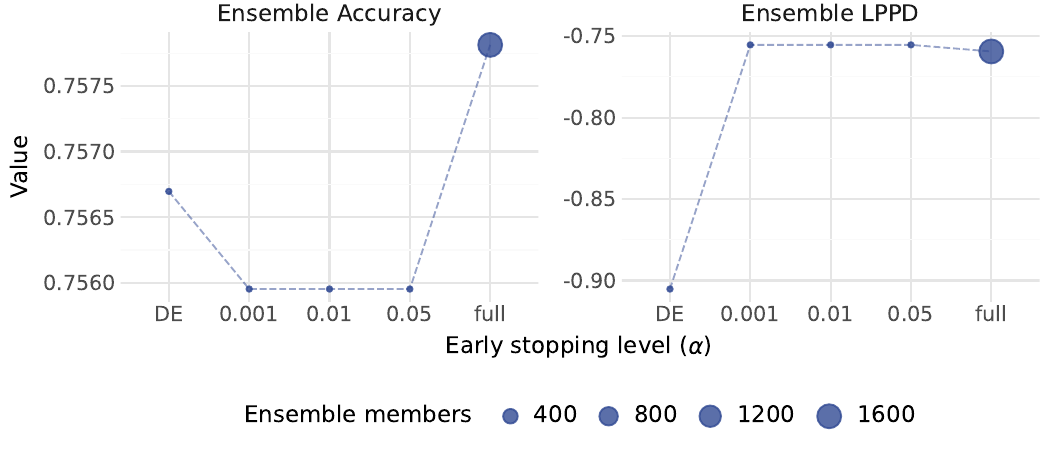}
        \caption{Ensemble performance (16 ensemble members)}
    \end{subfigure}
    \caption{Chainwise and ensemble hold-out test performance of the E-value induced minimal BDEs in comparison with the full ensemble and the optimized warmstarts i.e. the DE (members). The reference baseline is the optimized solutions \textbf{before} the sampler's warmup. The architecture is a 22M parameter ViT on the \texttt{Imagenette} dataset.}
    \label{fig:performance_vit_de}
\end{figure*}

\clearpage
\subsection{Additional Results for the ResNet-7}

\begin{figure*}[h]
    \centering
    \begin{subfigure}[t]{0.48\textwidth}
        \centering
        \includegraphics[width=\linewidth]{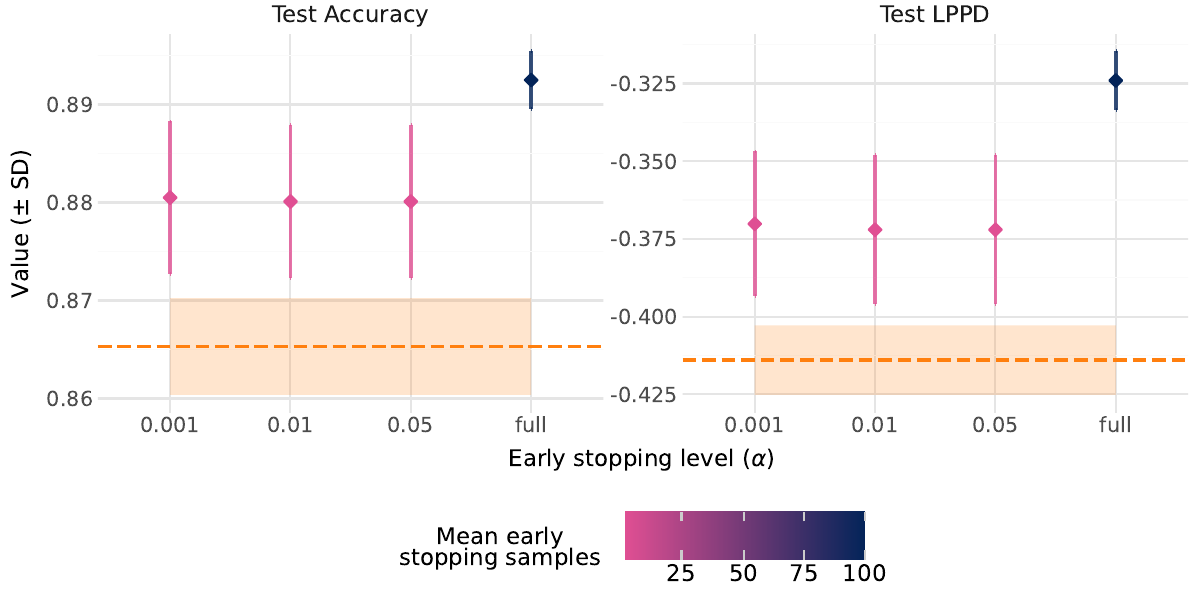}
        \caption{Chainwise performance (error bars represent standard\\deviation over 32 chains, optimized warmstart in orange)}
    \end{subfigure}%
    ~ 
    \begin{subfigure}[t]{0.48\textwidth}
        \centering
        \includegraphics[width=\linewidth]{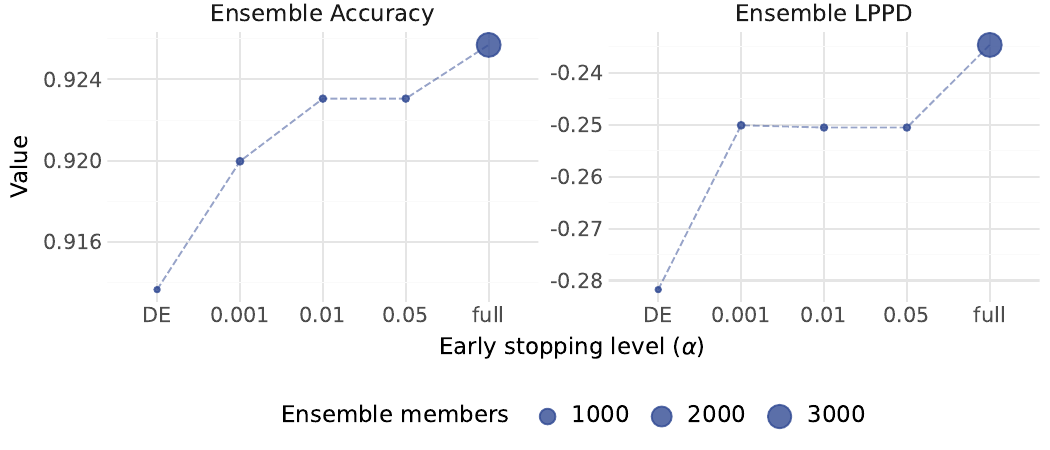}
        \caption{Ensemble performance (32 ensemble members)}
    \end{subfigure}
    \caption{Chainwise and ensemble hold-out test performance of the E-value induced minimal BDEs in comparison with the full ensemble and the optimized warmstarts i.e. the DE (members). The reference baseline is the optimized solutions \textbf{before} the sampler's warmup. The architecture is a ResNet-7 on the \texttt{CIFAR-10} dataset.}
    \label{fig:performance_resnet7_de}
\end{figure*}

\begin{table*}[h!]
\centering
\caption{Chainwise and ensemble hold-out test performance along with the compression and sizes of the E-value induced minimal BDEs in comparison with the full ensemble and the optimized warmstarts i.e. the DE (members). The architecture is a ResNet-7 on the \texttt{CIFAR-10} dataset and 32 chains/members are used.}
\label{tab:compression_resnet}
\resizebox{0.85\textwidth}{!}{
\begin{tabular}{c|l|cccc}
\toprule
 & \textbf{Method} & \textbf{(ensemble/mean) LPPD} & \textbf{(average) Samples} & \textbf{Compression vs. Full} \\
\midrule
\multirow{6}{*}{\parbox{3cm}{\centering\textbf{Ensemble}}} 
& DE & -0.2817 & 32 & $\times 100$\\
& $\alpha=0.01$ (DE ref.) & -0.2506 & 43 & $\times74.4$\\
& $\alpha=0.001$ (DE ref.) & -0.2501 & 50 & $\times64$\\
& $\alpha=0.01$ (Sample ref.) & -0.2489 & 312 & $\times10.3$\\
& $\alpha=0.001$ (Sample ref.) & -0.2484 & 322 & $\times9.9$\\
& Full & -0.2347 & 3200 & $\times 1$ \\
\midrule
\multirow{6}{*}{\parbox{3cm}{\centering\textbf{Single chain}}} 
& DNN & -0.4140 & 1 & $\times 100$\\
& $\alpha=0.01$ (DE ref.) & -0.3721 & 1.3 & $\times76.9$\\
& $\alpha=0.001$ (DE ref.) & -0.3702 & 1.6 & $\times62.5$\\
& $\alpha=0.01$ (Sample ref.) & -0.3521 & 9.8 & $\times10.2$\\
& $\alpha=0.001$ (Sample ref.) & -0.3524 & 10.1 & $\times9.9$\\
& Full & -0.3241 & 100 & $\times 1$ \\
\bottomrule
\end{tabular}}
\end{table*}

\subsection{Additional Results for the regression MLP}

\begin{figure*}[h]
    \centering
    \begin{subfigure}[t]{0.48\textwidth}
        \centering
        \includegraphics[width=\linewidth]{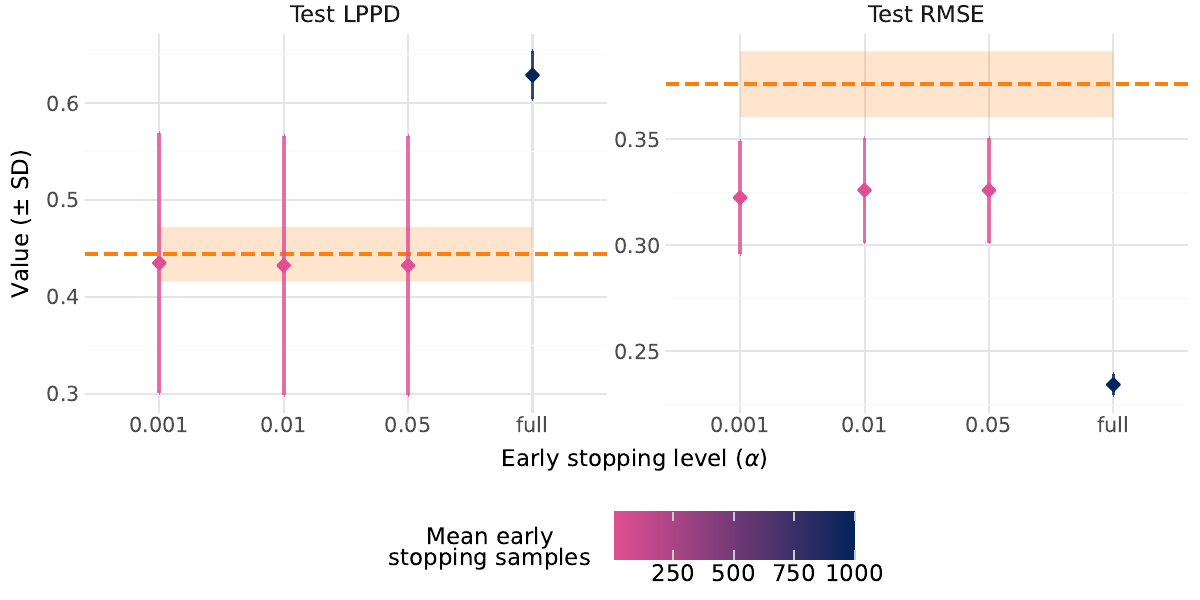}
        \caption{Chainwise performance (error bars represent standard\\deviation over 16 chains, optimized warmstart in orange)}
    \end{subfigure}%
    ~ 
    \begin{subfigure}[t]{0.48\textwidth}
        \centering
        \includegraphics[width=\linewidth]{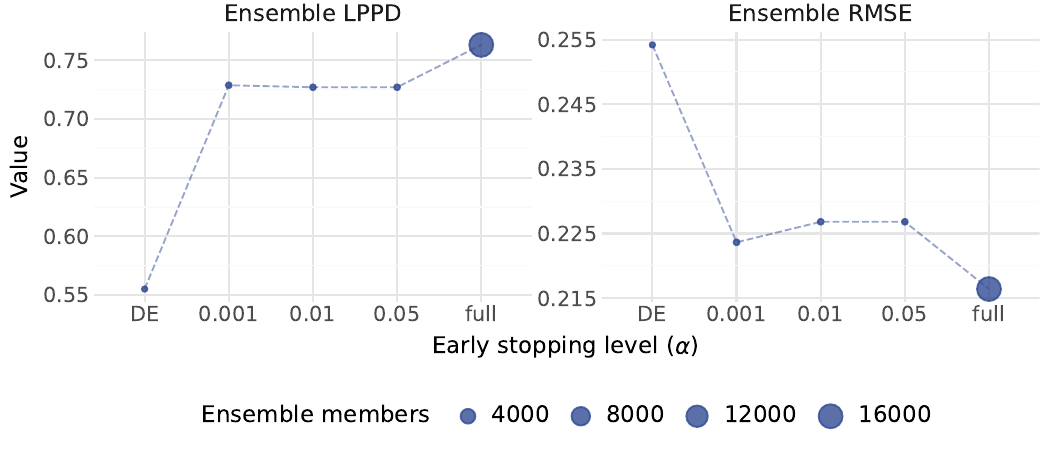}
        \caption{Ensemble performance (16 ensemble members)}
    \end{subfigure}
    \caption{Chainwise and ensemble hold-out test performance of the E-value induced minimal BDEs in comparison with the full ensemble and the optimized warmstarts i.e. the DE (members). The reference baseline is the optimized solutions \textbf{before} the sampler's warmup. The task is distributional regression via a 16x4 MLP on the \texttt{bikesharing} dataset.}
    \label{fig:performance_mlpbike_de}
\end{figure*}

\begin{figure*}[h]
    \centering
    \begin{subfigure}[t]{0.48\textwidth}
        \centering
        \includegraphics[width=\linewidth]{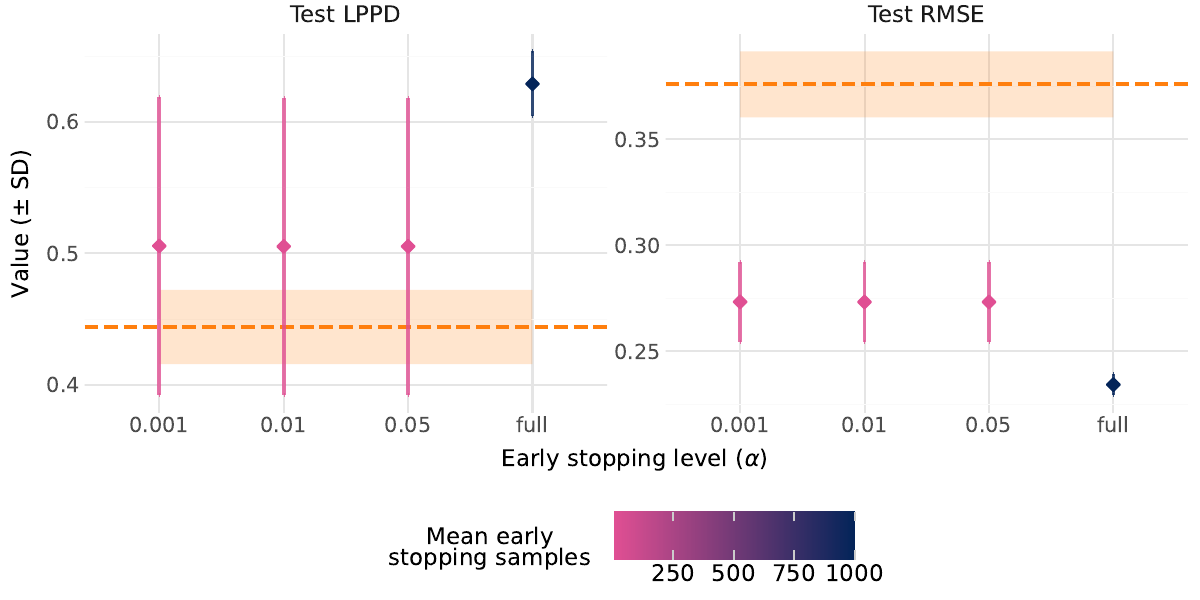}
        \caption{Chainwise performance (error bars represent standard\\deviation over 16 chains, optimized warmstart in orange)}
    \end{subfigure}%
    ~ 
    \begin{subfigure}[t]{0.48\textwidth}
        \centering
        \includegraphics[width=\linewidth]{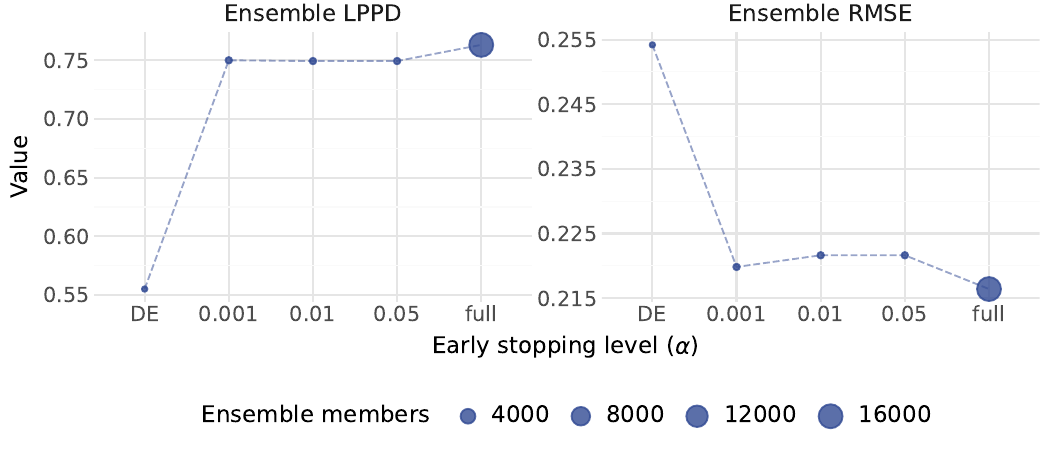}
        \caption{Ensemble performance (16 ensemble members)}
    \end{subfigure}
    \caption{Chainwise and ensemble hold-out test performance of the E-value induced minimal BDEs in comparison with the full ensemble and the optimized warmstarts i.e. the DE (members). The reference baseline is the first posterior sample \textbf{after} the sampler's warmup. The task is distributional regression via a 16x4 MLP on the \texttt{bikesharing} dataset.}
    \label{fig:performance_mlpbike_firstsample}
\end{figure*}

\begin{table*}[h!]
\centering
\caption{Chainwise and ensemble hold-out test performance along with the compression and sizes of the E-value induced minimal BDEs in comparison with the full ensemble and the optimized warmstarts i.e. the DE (members). The task is distributional regression via a 16x4 MLP on the \texttt{bikesharing} dataset and 16 chains/members are used.}
\label{tab:compression_mlpbike}
\resizebox{0.85\textwidth}{!}{
\begin{tabular}{c|l|cccc}
\toprule
 & \textbf{Method} & \textbf{(ensemble/mean) LPPD} & \textbf{(average) Samples} & \textbf{Compression vs. Full} \\
\midrule
\multirow{6}{*}{\parbox{3cm}{\centering\textbf{Ensemble}}} 
& DE & 0.5551 & 16 & $\times 1000$\\
& $\alpha=0.01$ (DE ref.) & 0.7270 & 25 & $\times640$\\
& $\alpha=0.001$ (DE ref.) & 0.7287 & 26 & $\times615.4$\\
& $\alpha=0.01$ (Sample ref.) & 0.7493 & 64 & $\times250$\\
& $\alpha=0.001$ (Sample ref.) & 0.7499 & 66 & $\times242.4$\\
& Full & 0.7631 & 16000 & $\times 1$ \\
\midrule
\multirow{6}{*}{\parbox{3cm}{\centering\textbf{Single chain}}} 
& DNN & 0.4440 & 1 & $\times 1000$\\
& $\alpha=0.01$ (DE ref.) & 0.4324 & 1.6 & $\times625$\\
& $\alpha=0.001$ (DE ref.) & 0.4348 & 1.6 & $\times625$\\
& $\alpha=0.01$ (Sample ref.) & 0.5053 & 4 & $\times250$\\
& $\alpha=0.001$ (Sample ref.) & 0.5057 & 4.1 & $\times243.9$\\
& Full & 0.6289 & 1000 & $\times 1$ \\
\bottomrule
\end{tabular}}
\end{table*}

\clearpage
\subsection{Additional Results for the Classification MLP}

\begin{figure*}[h]
    \centering
    \begin{subfigure}[t]{0.48\textwidth}
        \centering
        \includegraphics[width=\linewidth]{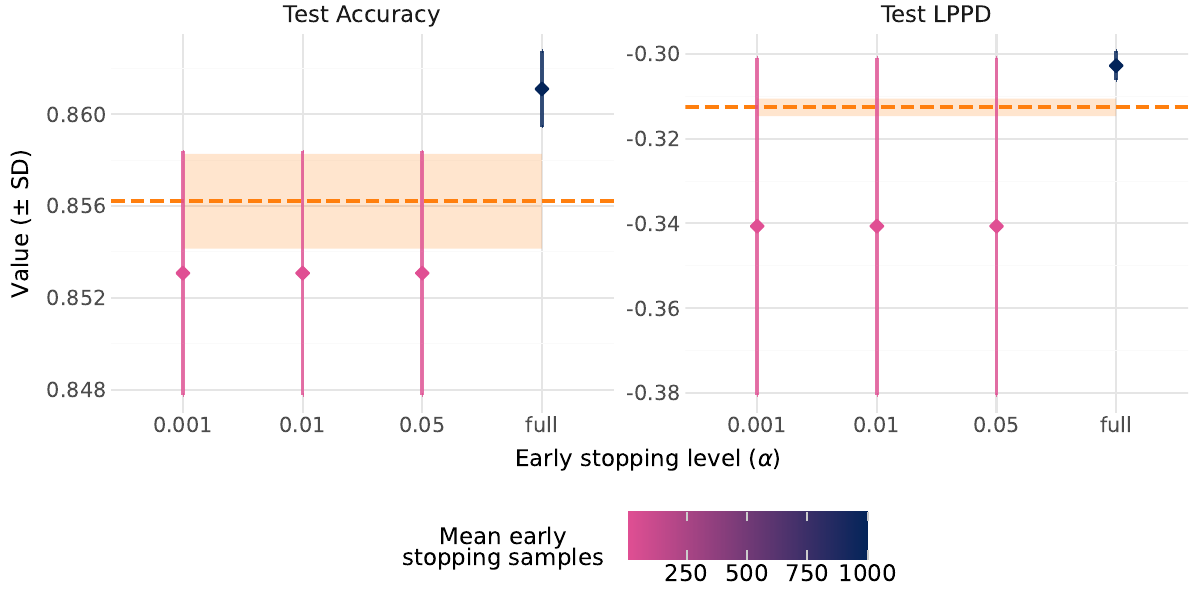}
        \caption{Chainwise performance (error bars represent standard\\deviation over 16 chains, optimized warmstart in orange)}
    \end{subfigure}%
    ~ 
    \begin{subfigure}[t]{0.48\textwidth}
        \centering
        \includegraphics[width=\linewidth]{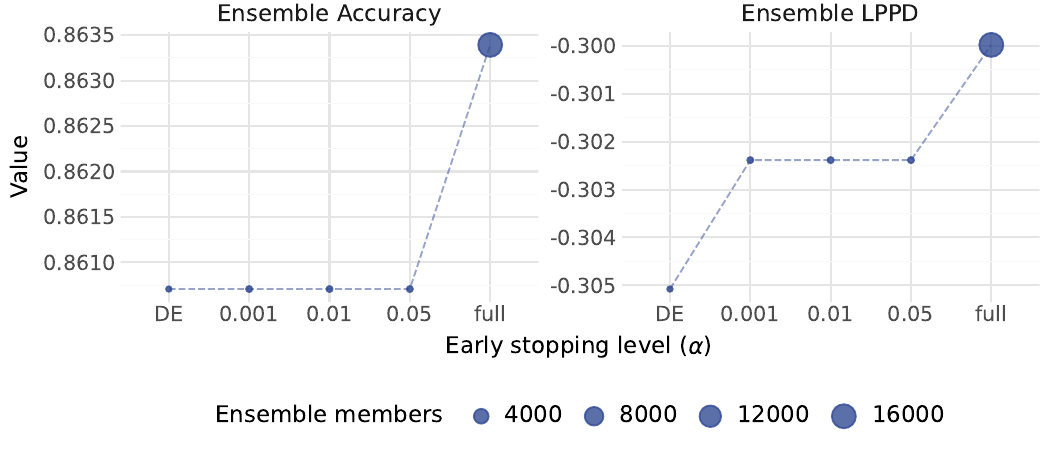}
        \caption{Ensemble performance (16 ensemble members)}
    \end{subfigure}
    \caption{Chainwise and ensemble hold-out test performance of the E-value induced minimal BDEs in comparison with the full ensemble and the optimized warmstarts i.e. the DE (members). The reference baseline is the first posterior sample \textbf{after} the sampler's warmup. The task is binary classification via a 16x3 MLP on the \texttt{income} dataset.}
    \label{fig:performance_mlpincome_firstsample}
\end{figure*}

\begin{table*}[h!]
\centering
\caption{Chainwise and ensemble hold-out test performance along with the compression and sizes of the E-value induced minimal BDEs in comparison with the full ensemble and the optimized warmstarts i.e. the DE (members). The task is binary classification via a 16x3 MLP on the \texttt{income} dataset and 16 chains/members are used.}
\label{tab:compression_mlpincome}
\resizebox{0.85\textwidth}{!}{
\begin{tabular}{c|l|cccc}
\toprule
 & \textbf{Method} & \textbf{(ensemble/mean) LPPD} & \textbf{(average) Samples} & \textbf{Compression vs. Full} \\
\midrule
\multirow{6}{*}{\parbox{3cm}{\centering\textbf{Ensemble}}} 
& DE & -0.3051 & 16 & $\times 1000$\\
& $\alpha=0.01$ (DE ref.) & - & 0 & - \\
& $\alpha=0.001$ (DE ref.) & - & 0 & - \\
& $\alpha=0.01$ (Sample ref.) & -0.3024 & 33 & $\times484.8$\\
& $\alpha=0.001$ (Sample ref.) & -0.3024 & 33 & $\times484.8$\\
& Full & -0.3000 & 16000 & $\times 1$ \\
\midrule
\multirow{6}{*}{\parbox{3cm}{\centering\textbf{Single chain}}} 
& DNN & -0.3126 & 1 & $\times 1000$\\
& $\alpha=0.01$ (DE ref.) & - & 0 & - \\
& $\alpha=0.001$ (DE ref.) & - & 0 & - \\
& $\alpha=0.01$ (Sample ref.) & -0.3406 & 2.1 & $\times476.2$\\
& $\alpha=0.001$ (Sample ref.) & -0.3406 & 2.1 & $\times476.2$\\
& Full & -0.3028 & 1000 & $\times 1$ \\
\bottomrule
\end{tabular}}
\end{table*}

\end{document}